\documentclass[runningheads,a4paper]{llncs}
\usepackage{amssymb}
\usepackage{graphicx}
\usepackage{rotating}
\usepackage[linesnumbered,ruled,vlined]{algorithm2e}
\usepackage{comment}
\usepackage{url}
\urldef{\mailsa}\path|giacomo.kahn@isima.fr|
\urldef{\mailsb}\path|contact@alexandrebazin.com|
\newcommand{\keywords}[1]{\par\addvspace\baselineskip
\noindent\keywordname\enspace\ignorespaces#1}
\usepackage{cite}
\usepackage{amsmath}
\usepackage{subfig}

\DeclareMathOperator{\argmax}{argmax} 

\usepackage{tikz}
\usetikzlibrary[arrows,patterns, decorations.pathreplacing]

\begin{document}

\mainmatter  

\title{Introducer Concepts in $n$-Dimensional Contexts}

\titlerunning{Introducer Concepts in $n$-Dimensional Contexts}

%
%
\author{Giacomo Kahn\inst{1} \and Alexandre Bazin\inst{2}}
\authorrunning{Kahn \& Bazin}

\institute{LIMOS \& Universit\'e Clermont Auvergne, France \\
\and
Le2i - Laboratoire Electronique, Informatique et Image, France\\
\mailsa, \mailsb\\
}

%
%

\toctitle{Lecture Notes in Computer Science}
\tocauthor{Authors' Instructions}
\maketitle

\begin{abstract}
Concept lattices are well-known conceptual structures that organise interesting patterns---the concepts---extracted from data.
In some applications, such as software engineering or data mining, the size of the lattice can be a problem, as it is often too large to be efficiently computed, and too complex to be browsed.
For this reason, the Galois Sub-Hierarchy, a restriction of the concept lattice to introducer concepts, has been introduced as a smaller alternative.
In this paper, we generalise the Galois Sub-Hierarchy to $n$-lattices, conceptual structures obtained from multidimensional data in the same way that concept lattices are obtained from binary relations.

\end{abstract}

\section{Introduction}

Formal Concept Analysis~\cite{DBLP:books/daglib/0095956} is a mathematical framework that allows, from a binary relation, to extract interesting patterns called concepts.
Those patterns form a hierarchy called a concept lattice.

\medskip
The size of the lattice, potentially exponential in the size of the relation, is one of the main drawbacks in its use as data representation.
In some cases, it is possible to avoid using the whole lattice, as it contains redundant information~\cite{DBLP:journals/tapos/GodinMMMAC98}.
The AOC-poset (or Galois Sub-Hierarchy) is a sub-order of the lattice that preserves only some of its key elements~\cite{Godin1993}.
Its size is potentially much smaller than the size of the associated concept lattice~\cite{DBLP:conf/icfca/CarbonnelHG15} and it can be used in place of the concept lattice to perform certain tasks~\cite{DBLP:journals/ijgs/DolquesBHG16, DBLP:conf/ismis/BazinCK17}.

\medskip
The generalisation of FCA to the $n$-dimensional case, Polyadic Concept Analysis~\cite{DBLP:journals/order/Voutsadakis02}, focuses on multidimensional data, i.e. $n$-ary relations. 

\medskip
In this paper, we generalise the notion of AOC-posets to $n$-lattices.
In Section~\ref{sec:def}, we provide the definitions and notations that we use throughout the paper.
Section~\ref{sec:introducers} is dedicated to the definition of introducer concepts in $n$-lattices, and some properties about those concepts.
In Section~\ref{sec:algo}, we present an algorithm to compute the introducer sub-order, and study its complexity.
Finally, we conclude and discuss some future works.

\section{Definitions and Notations}
\label{sec:def}

In this section, we introduce classical definitions and notations from Formal Concept Analysis and Polyadic Concept Analysis.
They can also be found in~\cite{DBLP:books/daglib/0095956} and~\cite{DBLP:journals/order/Voutsadakis02}.

\subsection{Formal Concept Analysis}

From now on, we will 
omit the brackets in the notation for sets when no confusion is induced by this simplification.

\medskip
A (formal) context is a triple $(\mathcal S_1,\mathcal S_2,\mathcal R)$ in which $\mathcal S_1$ and $\mathcal S_2$ are sets and $\mathcal R\subseteq \mathcal S_1\times \mathcal S_2$ is a binary relation between them.
The elements of $\mathcal S_1$ are called the (formal) objects and those of $\mathcal S_2$ the (formal) attributes.
A pair $(x_1,x_2)\in\mathcal R$ means that ``the object $x_1$ has the attribute $x_2$''.
A context can be represented as a cross table, as shown in Fig.~\ref{fig:toycontext}. For instance, object $1$ has attributes $a$ and $b$, and attribute $b$ is shared by objects $1$ and $2$.

\medskip
\begin{figure}[ht]
\centering
\begin{tabular}{c | c c c}
 & a & b & c \\
\hline
1 &$\times$&$\times$& \\
2 & & $\times$ & $\times$ \\
3 & $\times$ & & $\times$ \\
\end{tabular}
\caption{\label{fig:toycontext}An example of a context $\mathcal C =(\mathcal S_1, \mathcal S_2,\mathcal R)$ with $\mathcal S_1=\{1,2,3\}$ and $\mathcal S_2 =\{a, b, c\}$.}
\end{figure}

\medskip

Two \emph{derivation operators} $(\cdot)^\prime:2^{\mathcal S_1}\mapsto 2^{\mathcal S_2}$ and $(\cdot)^\prime:2^{\mathcal S_2}\mapsto 2^{\mathcal S_1}$ are defined. For $X_1\subseteq \mathcal S_1$ and $X_2\subseteq \mathcal S_2$, $X_1^\prime=\{a\ |\ \forall x\in X_1, (x,a)\in\mathcal R\}$ and $X_2^\prime =\{o\ |\ \forall y\in X_2, (o,y)\in \mathcal R\}$.

\medskip
A formal concept is a pair $(X_1,X_2)$ where $X_1\subseteq \mathcal S_1$, $X_2\subseteq \mathcal S_2$, $X_1' = X_2$ and $X_2' = X_1$.
This corresponds to a maximal set of objects that share a maximal set of attributes and can be viewed as a maximal rectangle full of crosses in the formal context, up to permutations on the elements of the rows and columns.
$X_1$ is called the \emph{extent} of the concept, while $X_2$ is called the \emph{intent}.

\medskip
The set of all concepts of a context ordered by the inclusion relation on either one of their components forms a complete lattice.
Additionally, every complete lattice is isomorphic to the concept lattice of some context~\cite{DBLP:books/daglib/0095956}. The concept lattice associated with the formal context from Fig.~\ref{fig:toycontext} is shown in Fig.~\ref{fig:toylattice}.

\begin{figure}[ht]
\centering
\begin{tikzpicture}[-,>=stealth',shorten >=1pt,shorten <=4pt]

  \node (bot) at (0,0) {$(\emptyset, abc)$};
  \node (c1) at (-1,1.3) {$(1,ab)$};
  \node (c2) at (0,1.3) {$(2,bc)$};
  \node (c3) at (1,1.3) {$(3,ac)$};
  \node (b1) at (-1,2.7) {$(12,b)$};
  \node (b2) at (0,2.7) {$(13,a)$};
  \node (b3) at (1,2.7) {$(23,c)$};
  \node (top) at (0,4) {$(123,\emptyset)$};
  \path[] 
  (bot) edge [] (c1)
  (bot) edge [] (c2)
  (bot) edge [] (c3)
  (top) edge [] (b1)
  (top) edge [] (b2)
  (top) edge [] (b3)
  (c1) edge [] (b1)
  (c1) edge [] (b2)
  (c2) edge [] (b1)
  (c2) edge [] (b3)
  (c3) edge [] (b2)
  (c3) edge [] (b3);

\end{tikzpicture}
\caption{Concept lattice associated with $\mathcal C$.\label{fig:toylattice}}
\end{figure}
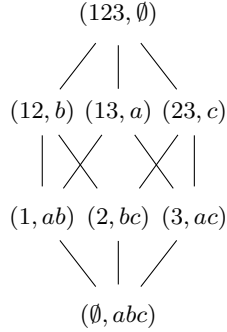

\medskip
In an application, the size of the concept lattice might be a drawback.
For this reason, Godin et al.~\cite{Godin1993} introduced a sub-hierarchy of the lattice, the Galois sub-hierarchy.
This sub-hierarchy was introduced and is most often used in the field of software engineering, but is also used in other fields, such as Relational Concept Analysis (RCA)~\cite{DBLP:journals/amai/HuchardHRV07} and data mining~\cite{DBLP:journals/ijgs/DolquesBHG16}.
Additionally, the Galois sub-hierarchy is integrated in some FCA tools, such as Latviz~\cite{DBLP:conf/cla/AlamLN16}, {\sc Galicia}~\cite{valtchev2003galicia}, RCAExplore~\footnote[1]{\url{http://dolques.free.fr/rcaexplore/}} or AOC-poset Builder~\footnote[2]{\url{http://www.lirmm.fr/AOC-poset-Builder/}}.

\medskip
\begin{definition}[Introducer concept]
An \emph{Object-Concept} is a concept $(o^{\prime\prime}, o^\prime)$ with $o\in \mathcal S_1$.
We say that this concept \emph{introduces} $o$.

An \emph{Attribute-Concept} is a concept $(a^\prime, a^{\prime\prime})$ with $a\in \mathcal S_2$.
We say that this concept \emph{introduces} $a$.
\end{definition}

\medskip
A concept can introduce both attributes and objects, or it can introduce neither.
We call the sub-order restricted to the introducer concepts an \emph{Attribute-Object-Concept partially ordered set} (AOC-poset), or Galois Sub-Hierarchy (GSH).
While a concept lattice can have up to $2^{min(|\mathcal S_1|,|\mathcal S_2|)}$ concepts, the associated GSH has at most $|\mathcal S_1|+|\mathcal S_2|$ elements.
Several algorithms exist to compute the GSH~\cite{Dicky1995, DBLP:journals/ita/HuchardDL00, DBLP:conf/icfca/ArevaloBHPS07, Berry2014}.

\subsection{Polyadic Concept Analysis}

\begin{definition}

An \emph{$n$-context} is an $(n+1)$-tuple $\mathcal{C} = (\mathcal S_1,\dots,\mathcal S_n,\mathcal R)$ in which $S_i$, $i\in \{1,\dots,n\}$, is a set called a \emph{dimension} and $R$ is an $n$-ary relation between the dimensions.

\end{definition}

\medskip
An $n$-context can be represented by a $|S_1|\times \dots \times |S_n|$ cross table as illustrated in Fig. \ref{fig:context}.

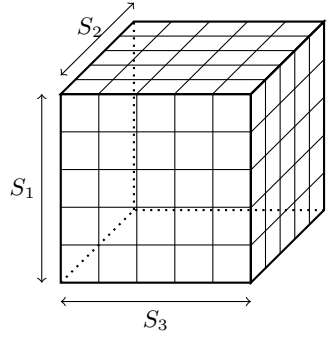
\begin{figure}[ht]
\centering
\begin{tikzpicture}[scale=0.5]
\pgfmathsetmacro{\cubex}{5}
\pgfmathsetmacro{\cubey}{5}
\pgfmathsetmacro{\cubez}{5}
\draw[thick] (0,0,0) -- ++(-\cubex,0,0) -- ++(0,-\cubey,0) -- ++(\cubex,0,0) -- cycle;
\draw[thick] (0,0,0) -- ++(0,0,-\cubez) -- ++(0,-\cubey,0) -- ++(0,0,\cubez) -- cycle;
\draw[thick] (0,0,0) -- ++(-\cubex,0,0) -- ++(0,0,-\cubez) -- ++(\cubex,0,0) -- cycle;
\draw[dotted, thick] (-\cubex,-\cubey,0) -- ++(0,0,-\cubez) -- ++ (\cubex,0,0);
\draw[dotted, thick] (-\cubex,-\cubey,-\cubez) -- ++(0,\cubey,0);
\foreach \x in {1,...,\cubex}{
	\draw (-\cubex+\x, -\cubey, 0) -- ++ (0,\cubey, 0);
	\draw (-\cubex-0.01+\x, 0, 0) -- ++ (0,0,-\cubez);
}
\foreach \y in {1,...,\cubey}{
	\draw (-\cubex, -\cubey+\y, 0) -- ++ (\cubex,0, 0);
	\draw (0, -\cubey-0.01+\y, 0) -- ++ (0, 0, -\cubez);
}
\foreach \z in {1,...,\cubez}
	\draw (-\cubex, 0, -\z)-- ++ (\cubex,0, 0) -- ++ (0,-\cubey,0);

\node (S1) at (-\cubex-1, -2.5,0) {$S_1$};
\coordinate (boutS11) at (-\cubex-0.5, 0, 0);
\coordinate (boutS12) at (-\cubex-0.5, -\cubey, 0);
\draw[<->] (boutS11) -- (boutS12);
\node (S2) at (-\cubex, 1,-2) {$S_2$};
\coordinate (boutS21) at (-\cubex, 0.5, 0);
\coordinate (boutS22) at (-\cubex, 0.5, -\cubez);
\draw[<->] (boutS21) -- (boutS22);
\node (S3) at (-2.5, -\cubey-1,0) {$S_3$};
\coordinate (boutS31) at (-\cubex, -\cubey-0.5, 0);
\coordinate (boutS32) at (0, -\cubey-0.5, 0);
\draw[<->] (boutS31) -- (boutS32);
\end{tikzpicture}
\caption{Visual representation of a 3-context without its crosses\label{fig:context}.}
\end{figure}

\medskip
\begin{definition}

\medskip
An \emph{$n$-concept} of $\mathcal{C} = (S_1,\dots,S_n,R)$ is an $n$-tuple $(X_1,\dots,X_n)$ such that~$\prod_{i\in \{1,\dots,n\}} X_i\subseteq R$ and there are no $i\in \{1,\dots,n\}$ and $k\in S_i\setminus X_i$ such that~$\{k\}\times \prod_{j\in \{1,\dots,n\}\setminus \{i\}} X_j\subseteq R$.

\medskip
\end{definition}

An $n$-concept can be viewed as a maximal $n$-dimensional box full of crosses up to permutations on the elements of the dimensions.
We denote by $\mathcal T(\mathcal C)$ the set of $n$-concepts of a $n$-context $\mathcal C$.

\medskip
\begin{figure}[ht]

\begin{center}
\begin{tabular}{c|lll||lll}
 & a & b & c & a & b & c \\
\hline
1 &$\times$&$\times$&\hphantom{$\times$}& $\times$ & & \\

2 &  & & & $\times$ & & \\

3 & $\times$ & & & $\times$ & & $\times$\\
\hline
\multicolumn{1}{c|}{} & \multicolumn{3}{c||}{$\alpha$} & \multicolumn{3}{c}{$\beta$}\\

\end{tabular}
\caption{An example of a $2\times 3\times 3$ $3$-context\label{fig:exampleContext}.}
\end{center}
\end{figure}

\medskip
In the Fig.~\ref{fig:exampleContext} example, seven 3-concepts are present: $(\alpha, 1, ab)$, $(\alpha\beta, 13, a)$, $(\beta, 3, ac)$, $(\beta, 123, a)$, $(\alpha\beta, 123, \emptyset)$, $(\alpha\beta,\emptyset,abc)$ and $(\emptyset, 123,abc)$.

\medskip
\begin{definition}[From~\cite{DBLP:journals/order/Voutsadakis07}]
$\mathcal S=(S,\lesssim_1,\dots,\lesssim_n)$ is an $n$-ordered set if for $A\in S$ and $B\in S$ :
\begin{enumerate}
\item $A\sim_i B,\forall i\in\{1,\dots,n\} \Rightarrow A=B$ (Uniqueness Condition)
\item $A\lesssim_i B, \forall i\in(\{1,\dots,n\}\setminus j)\Rightarrow B\lesssim_j A$ (Antiordinal Dependency)
\end{enumerate}
\end{definition}

\medskip
For the Antiordinal Dependency condition to be respected, it is sufficient to have $i,j\in \{1,\dots,n\}, i\not=j$ such that $A\lesssim_i B$ and $B\lesssim_j A$.

\medskip
The set of all the $n$-concepts of an $n$-context together with the $n$ quasi-orders $\lesssim_i$ induced by the inclusion relations on the subsets of each dimension forms an $n$-ordered set.
Additionally, the existence of some particular joins makes it a complete $n$-lattice.
Every $n$-lattice can be associated with some $n$-context~\cite{DBLP:journals/order/Voutsadakis02}. 

\medskip
\begin{definition}\label{def:couche}
Let $x\in \mathcal S_i$ be an element of a dimension $i$.
We denote by $\mathcal C_x$ the $(n-1)$-context $\mathcal C_x = (S_1,\dots,S_{i-1},S_{i+1},\dots,S_n, \mathcal R_x)$ where \[\mathcal R_x=\{(s_1,\dots,s_{i-1}, s_{i+1},\dots,s_n)\ |\ (s_1,\dots,s_{i-1},x,s_{i+1},\dots,s_n)\in \mathcal R\}\]
\end{definition}

With the previous definition, $C_x$ is the $(n-1)$-context corresponding to element $x$, represented by the shaded area in Fig.~\ref{fig:couche}.
\begin{figure}[ht]
\centering
\begin{tikzpicture}[scale=0.53]
\pgfmathsetmacro{\cubex}{5}
\pgfmathsetmacro{\cubey}{5}
\pgfmathsetmacro{\cubez}{5}
\draw[thick] (0,0,0) -- ++(-\cubex,0,0) -- ++(0,-\cubey,0) -- ++(\cubex,0,0) -- cycle;
\draw[thick] (0,0,0) -- ++(0,0,-\cubez) -- ++(0,-\cubey,0) -- ++(0,0,\cubez) -- cycle;
\draw[thick] (0,0,0) -- ++(-\cubex,0,0) -- ++(0,0,-\cubez) -- ++(\cubex,0,0) -- cycle;
\draw[dotted, thick] (-\cubex,-\cubey,0) -- ++(0,0,-\cubez) -- ++ (\cubex,0,0);
\draw[dotted, thick] (-\cubex,-\cubey,-\cubez) -- ++(0,\cubey,0);
\foreach \x in {1,...,\cubex}{
	\draw (-\cubex+\x, -\cubey, 0) -- ++ (0,\cubey, 0);
	\draw (-\cubex-0.01+\x, 0, 0) -- ++ (0,0,-\cubez);
}
\foreach \y in {1,...,\cubey}{
	\draw (-\cubex, -\cubey+\y, 0) -- ++ (\cubex,0, 0);
	\draw (0, -\cubey-0.01+\y, 0) -- ++ (0, 0, -\cubez);
}
\foreach \z in {1,...,\cubez}
	\draw (-\cubex, 0, -\z)-- ++ (\cubex,0, 0) -- ++ (0,-\cubey,0);

\node (x) at (-0.25,0.25,-\cubez) {$x$};
\fill[color=gray!20, pattern=north east lines, very thin] (0,0,-\cubez) -- ++ (-1,0,0) -- ++ (0,0,\cubez) -- ++ (0,-\cubey,0) -- ++ (1,0,0) -- ++ (0,0,-\cubez) --cycle;

\node (S1) at (-\cubex-1, -2.5,0) {$S_1$};
\coordinate (boutS11) at (-\cubex-0.5, 0, 0);
\coordinate (boutS12) at (-\cubex-0.5, -\cubey, 0);
\draw[<->] (boutS11) -- (boutS12);
\node (S2) at (-\cubex, 1,-2) {$S_2$};
\coordinate (boutS21) at (-\cubex, 0.5, 0);
\coordinate (boutS22) at (-\cubex, 0.5, -\cubez);
\draw[<->] (boutS21) -- (boutS22);
\node (S3) at (-2.5, -\cubey-1,0) {$S_3$};
\coordinate (boutS31) at (-\cubex, -\cubey-0.5, 0);
\coordinate (boutS32) at (0, -\cubey-0.5, 0);
\draw[<->] (boutS31) -- (boutS32);
\end{tikzpicture}
\caption{If $x$ is an element of $\mathcal S_3$ in this 3-context, then $C_x$ is the 2-context resulting from fixing $x$\label{fig:couche}.}
\end{figure}
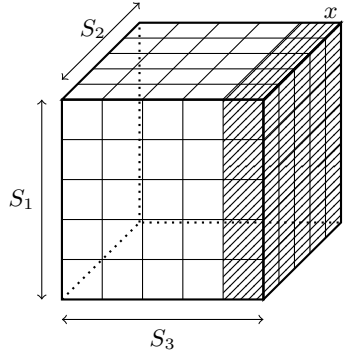

\section{Introducer concepts in $n$-Lattices}
\label{sec:introducers}

In this section, we define introducer concepts in $n$-lattices. In the next definitions, we will call dimension $i$ the height, while all other dimensions are called the width.

\medskip
\begin{definition}
Let $x\in\mathcal S_i$ be an element of a dimension $i$.
The concepts with maximal width such that $x$ is in the height are the introducer concepts of $x$.
The set of introducer concepts of $x$ is denoted by $I_x$.
\end{definition}

\medskip
In the Fig.~\ref{fig:exampleContext} example, we have $I_\alpha=\{(\alpha\beta,13,a), (\alpha,1,ab)\}$.

\medskip
We denote by $\mathcal I(\mathcal S_i) = \bigcup_{x\in\mathcal S_i} I_x$ the set of concepts that introduce an element of dimension $i$ and by $\mathcal I(\mathcal C) = \bigcup_{i\in\{1,\dots,n\}} \mathcal I(\mathcal S_i)$ the set of all introducer concepts of a context $\mathcal C$.

\medskip
As in the 2-dimensional case, irreducible elements are introducer concepts.
However, $\mathcal I(\mathcal C)$ is not always strictly the set of irreducible elements as some applications expect the context not to be reduced.

\medskip
\begin{proposition}\label{prop:nposet}
$(\mathcal I(\mathcal C),\lesssim_1,\dots,\lesssim_n)$ is an $n$-ordered set.
\end{proposition}
\begin{proof}
Let $A$ and $B$ be in $\mathcal I(\mathcal C)$. We recall that $A_i\subseteq B_i\Leftrightarrow A\lesssim_iB$ and that $A_i=B_i\Leftrightarrow A\sim_i B$.
Without loss of generality, let $A\in\mathcal I(\mathcal S_i)$ and $B\in\mathcal I(\mathcal S_j)$.

If $\forall k\in\{1,\dots,n\}$, $A\sim_k B$, then  $\forall k\in\{1,\dots,n\}$, $A_k = B_k$, so $A=B$ (Uniqueness Condition).

If $A$ and $B$ are distinct, $\exists k\in\{1,\dots,n\}$ such that $A\lesssim_k B$ or $B\lesssim_k A$.
Without loss of generality, suppose $A\lesssim_k B$.
Suppose that there is no $\ell\in\{1,\dots,n\}$ such that $B\lesssim_\ell A$.
That implies that all the components of $A$ are included in the components of $B$.
Then this is in contradiction with the maximality condition implied by $A$ being a concept.
Thus $\exists j\in\{1,\dots,n\}\setminus i$ such that $B\lesssim_j A$ (Antiordinal Dependency).
\qed
\end{proof}

\medskip
As in the 2-dimensional case where concept lattices and GSH are respectively complete lattices and partially ordered sets, in the $n$-dimensional case we have complete $n$-lattices and $n$-ordered sets.

\medskip
\begin{proposition}\label{prop:fromConcepts}


Let $x\in S_i$. If $(X_1,\dots,X_{i-1},X_{i+1},\dots,X_n)$ is an $(n-1)$-concept of $\mathcal{C}_x$, then $(X_1,\dots,\{x\}\cup X_i,\dots,X_n)$ is an introducer of $x$. 

If $(X_1,\dots,\{x\}\cup X_i,\dots,X_n)$ is an introducer of $x$, then there exists an $(n-1)$-concept $(X_1,\dots,X_{i-1},X_{i+1},\dots,X_n)$ in $\mathcal{C}_x$.

\end{proposition}

\begin{proof}
We suppose, without loss of generality, that $x\in \mathcal S_1$. The $(n-1)$-concepts of $\mathcal C_x$ are of the form $(X_2,\dots,X_n)$.
If $(x,X_2,\dots,X_n)$ is an $n$-concept of $\mathcal C$, then it is minimum in height and maximal in width and is thus an introducer of $x$.

\medskip
If $(x,X_2,\dots,X_n)$ is not an $n$-concept of $\mathcal C$, as $(X_2,\dots,X_n)$ is a $(n-1)$-concept of $\mathcal C_x$, then $(x,X_2,\dots,X_n)$ can be augmented only on the first dimension.
As such, there exists an $n$-concept $(\{x\}\cup X_1,X_2,\dots,X_n)$ that is maximal in width and is thus an introducer of $x$.

\medskip
Suppose that there is an $X=(X_1,\dots,X_n)\in I_x, x\in X_1$, that is not obtained from an $(n-1)$-concept of $\mathcal C_x$ by extending $X_1$.
It means that $(X_2,\dots,X_n)$ is not maximal in $\mathcal C_x$ (else it would be an $(n-1)$-concept).
Then, there exists an $n$-concept $Y=(Y_1,Y_2\dots,Y_n)$ with $Y_1\subseteq X_1$ and $X_i\subseteq Y_i$, for $i\in\{2,\dots,n\}$.
This is a contradiction with the fact that $X$ is an introducer of $x$.\qed
\end{proof}

Proposition~\ref{prop:fromConcepts} states that every $(n-1)$-concept of $\mathcal C_x$ maps to an introducer of $x$ in $\mathcal{C}$, and that every introducer of $x$ is the image of an $(n-1)$-concept of $\mathcal C_x$.

%

\section{Algorithm}
\label{sec:algo}

In this section, we present an algorithm to compute the introducer concepts in an $n$-context.

\medskip
Algorithm~\ref{algo:introDim} computes the introducers for each element of a dimension $i$. For a given element $x\in \mathcal S_i$, we compute $\mathcal T(\mathcal C_x)$. 
Then, for each $(n-1)$-concept $X\in \mathcal T(\mathcal C_x)$, we build the set $X_i$ needed to extend $X$ into an $n$-concept.
An element $y$ is added to $X_i$ when $y\times\prod_{j\not=i} X_j\subseteq \mathcal R$, that is if there exists an $(n-1)$-dimensional box full of crosses (but not necessarily maximal) in $\mathcal R$, at level $y$. The final set $X_i$ always contains at least $x$.

\medskip
\begin{algorithm}[ht]
\DontPrintSemicolon 
\KwIn{$\mathcal C$ an $n$-context, $i\in\{1,\dots,n\}$ a dimension}
\KwOut{$\mathcal I(\mathcal S_i)$ the set of introducer concepts of elements of dimension $i$}

$I\gets \emptyset$\;
\ForEach{$x\in\mathcal S_i$}{
	$C\gets\emptyset$\;
	\ForEach{$X=(X_1,\dots,X_{i-1},X_{i+1},\dots,X_n)\in\mathcal T(\mathcal C_x)$}{
		$X_i\gets \emptyset$\;\label{candidate}
        \ForEach{$y\in\mathcal S_i$}{\label{loop:extend}
        	\If{$\prod_{j\not=i} X_j\times y\subseteq \mathcal R$}{
				$X_i\gets X_i\cup y$\;
			}
		}
        $C\gets C \cup (X_1,\dots,X_i,\dots,X_n)$\;\label{line:cup1}
	}
    $I\gets I\cup C$\;\label{line:cup2}
}
\Return{$I$}\;
\caption{{\sc IntroducerDim}$(\mathcal C, i)$}
\label{algo:introDim}
\end{algorithm}

\medskip
Algorithm~\ref{algo:intro} calls Algorithm~\ref{algo:introDim} on each dimension. This ensure that each element of each dimension will be scanned for its introducer concepts. Algorithm~\ref{algo:intro} computes the introducer set $\mathcal I(\mathcal C)$ for $n$-context $\mathcal C$.

\medskip
\begin{algorithm}[ht]
\DontPrintSemicolon 
\KwIn{$\mathcal C$ an $n$-context}
\KwOut{$\mathcal I(\mathcal C)$ the set of all introducer concepts for $\mathcal C$}
$R \gets \emptyset$\;
\ForEach{dimension $i$} {
	$R\gets R\cup \text{\sc IntroducerDim}(\mathcal C,i)$\;\label{line:cup3}
}
\Return{$R$}\;
\caption{{\sc Introducers}$(\mathcal C)$}
\label{algo:intro}
\end{algorithm}

\medskip
Algorithm~\ref{algo:introDim} requires the computation of the $(n-1)$-concepts from an $(n-1)$-context. Several algorithms exist to complete this task~\cite{DBLP:journals/tkdd/CerfBRB09, makhalova2017incremental, bazin2017incremental}.
\medskip
\begin{proposition}
Algorithm~\ref{algo:introDim} ends and returns all the introducer concepts of elements of the dimension $\mathcal S_i$.
\end{proposition}
\begin{proof}
The $\mathcal S_i$ are finite.
The set of $(n-1)$-concepts of an $(n-1)$-context resulting from fixing an element is also finite.
Algorithm~\ref{algo:introDim} passes through each element $x\in\mathcal S_i$ and on each concept of $\mathcal T(\mathcal C_x)$ exactly once.
The maximality test on dimension $i$ looks at the elements of $\mathcal S_i$, which is finite.
Thus, the algorithm ends.

Proposition~\ref{prop:fromConcepts} ensures that every introducer of $x$ can be computed from the concepts of $\mathcal C_x$.
Thus, every introducer of an element of the dimension $S_i$ is returned.\qed
\end{proof}

\medskip
At the time of writing, the only known bound for the number of $n$-concepts of an $n$-context $(S_1,\dots,S_n,R)$ is $\prod_{i\in\{1,\dots,n\}\setminus k} 2^{|\mathcal S_i|}$ with $k=\argmax_{k\in\{1,\dots,n\}}|\mathcal S_k|$.
Let $\mathbb K_n$ be the maximal number of $n$-concepts in an $n$-context.
Computing $\mathcal C_x$ from $\mathcal C$ is in $O(|\mathcal R|)$.
Building the set $X_i$ that extends an $(n-1)$-concept of $\mathcal C_x$ into an introducer of $x$ can be done in $O(|\mathcal S_i|\times\prod_{j\not=i}|X_j|)$.
We denote by $T$ the complexity of computing $\mathcal T(\mathcal C_x)$ from $\mathcal C_x$.

Thus the complexity of Algorithm~\ref{algo:introDim} for context $\mathcal C =(\mathcal S_1,\dots,\mathcal S_n,\mathcal R)$ and dimension $i$ is $O\left(|\mathcal S_i|\times \left(T+\mathbb K_{n-1} \times \prod_{j\in\{1,\dots,n\}}|\mathcal S_j|\right)\right)$ and the complexity of Algorithm~\ref{algo:intro} is $O\left(\sum_{i\in\{1,\dots,n\}}\left(|\mathcal S_i|\times \left(T+\mathbb K_{n-1} \times \prod_{j\in\{1,\dots,n\}}|\mathcal S_j|\right)\right)\right)$.

\section{Conclusion}

In this paper, we introduced the $n$-dimensional equivalent of Galois Sub-Hierarchies or AOC-posets.
We showed that the set of introducer concepts, together with the $n$ quasi-orders induced by the inclusion on each dimension, forms an $n$-ordered set.
We provided an algorithm to compute the set of introducer concepts from an $n$-context.

\medskip
Although our approach was not initially motivated by an applicative problem, it would be interesting to use the notion of introducer concepts in $n$-dimensions to address some specific problems in software engineering or data mining.

\medskip
It would be interesting to experiment on datasets (real and generated) to evaluate the gains (in term of number of concepts) of the restriction to introducer concepts.

\section*{Acknowledgements}
This research was partially supported by the European Union's ``\emph{Fonds Europ\'een de D\'eveloppement R\'egional (FEDER)}'' program.

\bibliographystyle{unsrt}
\bibliography{Biblio}

\end{document}